\documentclass{article}

\usepackage{PRIMEarxiv}

\usepackage[utf8]{inputenc} 
\usepackage[T1]{fontenc}    
\usepackage{hyperref}       
\usepackage{url}            
\usepackage{booktabs}       
\usepackage{amsfonts}       
\usepackage{nicefrac}       
\usepackage{microtype}      
\usepackage{lipsum}
\usepackage{fancyhdr}       
\usepackage{graphicx}       
\graphicspath{{media/}}     

\usepackage{amsmath,amsthm}
\usepackage{bm}
\usepackage{mathtools}
\usepackage{amssymb}
\usepackage{longtable}
\usepackage{comment}
\usepackage{tikz}
\usepackage{xcolor}
\usepackage{algorithm}
\usepackage[noend]{algpseudocode}
\usepackage{amsmath,amssymb} 
\usepackage{comment}
\usepackage{color}
\usepackage{todonotes}

\newtheorem{definition}{Definition}
\newtheorem{proposition}{Proposition}

\title{\Large Enumeration-Guided Training of\\ Neural Hamilton-Jacobi Safety Certificates}

\author{
  Mihaela-Larisa Clement$^{1,2}$, Tosca Pierro$^{3}$, Ezio Bartocci$^{1}$ and Luca Marzari$^{1}$
  \\
  $1$ TU Wien, Vienna, Austria\\
  $2$ AIT Austrian Institute of Technology, Vienna, Austria\\
  $3$ Università Politecnica delle Marche, Italy\\
  Contact authors: \textit{mihaela-larisa.clement@tuwien.ac.at}, \textit{luca.marzari@tuwien.ac.at}.}

\begin{document}
\maketitle

\begin{abstract}
Hamilton–Jacobi (HJ) reachability analysis offers formal safety guarantees for control systems, but training accurate neural certificates remains difficult, particularly when unsafe regions are hard to characterize. We propose a verification-guided training framework that addresses this by using probabilistic enumeration of safety constraints to detect forward-invariance violations, i.e., states where the learned certificate incorrectly classifies unsafe behavior. This produces a compact, targeted set of counterexamples that more precisely captures the failure boundary than inputs generated by standard adversarial attack methods. Crucially, using this representation to fine-tune a pretrained HJ value function yields larger certified regions and reduces certification time across all tested benchmarks. We further introduce a shielding mechanism based on the refined certificates to enforce safety constraints at deployment time. We evaluate the proposed safety filter in a realistic, physics-based Isaac Sim environment and through real-world autonomous racing experiments, achieving zero safety constraint violations compared to a nominal controller and a backup policy based on Linear Quadratic Regulation (LQR). These results demonstrate the effectiveness of the proposed shielding mechanism in ensuring safe operation under high-speed, safety-critical conditions.
\end{abstract}


\section{Introduction}
For control systems interacting with the physical world, safety is a primary concern in controller design, as controllers must satisfy safety constraints at all time steps over an infinite horizon. Consider, for example, a robot locomotion task where safety requires maintaining a positive clearance from surrounding obstacles at all times. Two distinct challenges arise in such settings. The first is identifying the set of states from which long-term constraint satisfaction is guaranteed, the so-called \textit{feasible region}, since the deployment of a controller should be restricted to this set. The second is actively enforcing safety at runtime, as even when a controller is initialized within the feasible region, it may issue actions that push the system toward constraint violations, and an explicit safety mechanism is required to intercept and correct such actions before they are executed. Addressing both challenges simultaneously is essential for deploying safety-critical controllers in realistic environments.

A widely adopted approach to enforce safety constraints is based on Model Predictive Control (MPC)~\cite{rawlings2017model}. Nominal MPC incorporates state and input constraints directly into a finite-horizon optimal control problem solved repeatedly online. Under suitable assumptions, nominal MPC can exhibit a degree of robustness; however, when hard state constraints are present, these safety margins can become very small. To account for bounded disturbances and model uncertainty, robust MPC variants such as tube-based methods~\cite{10.1016/j.automatica.2004.08.019,Chisci2001} confine the system trajectory to a tube around the nominal trajectory, ensuring constraint satisfaction through tightening of the nominal problem. Nonetheless, constraint tightening can introduce conservatism that degrades performance in demanding tasks, and the finite planning horizon limits the scope of safety guarantees that can be provided.

A more principled framework for characterizing safety over an infinite horizon is Hamilton–Jacobi (HJ) reachability analysis~\cite{HJ}, which offers a formal setting for computing feasible regions of control systems subject to safety constraints. HJ reachability defines the feasible region as the zero-sublevel set of a value function $V(x)$, representing the worst-case constraint violation along an optimally controlled trajectory originating from state $x$. Compared with MPC-based approaches, HJ reachability provides direct global safety guarantees rather than relying on finite-horizon optimization and repeated replanning. This makes $V$ a natural candidate not only for certifying the feasible region, but also for constructing a runtime safety filter, namely a \textit{shield}, that corrects unsafe control actions by substituting a verified backup policy whenever a nominal controller would violate the safe set boundary. However, this increased rigor comes at a high computational cost, as the underlying HJ partial differential equations suffer from the well-known curse of dimensionality, making exact solutions intractable for high-dimensional systems.
\begin{figure*}[!t]
    \centering
    \includegraphics[width=\linewidth]{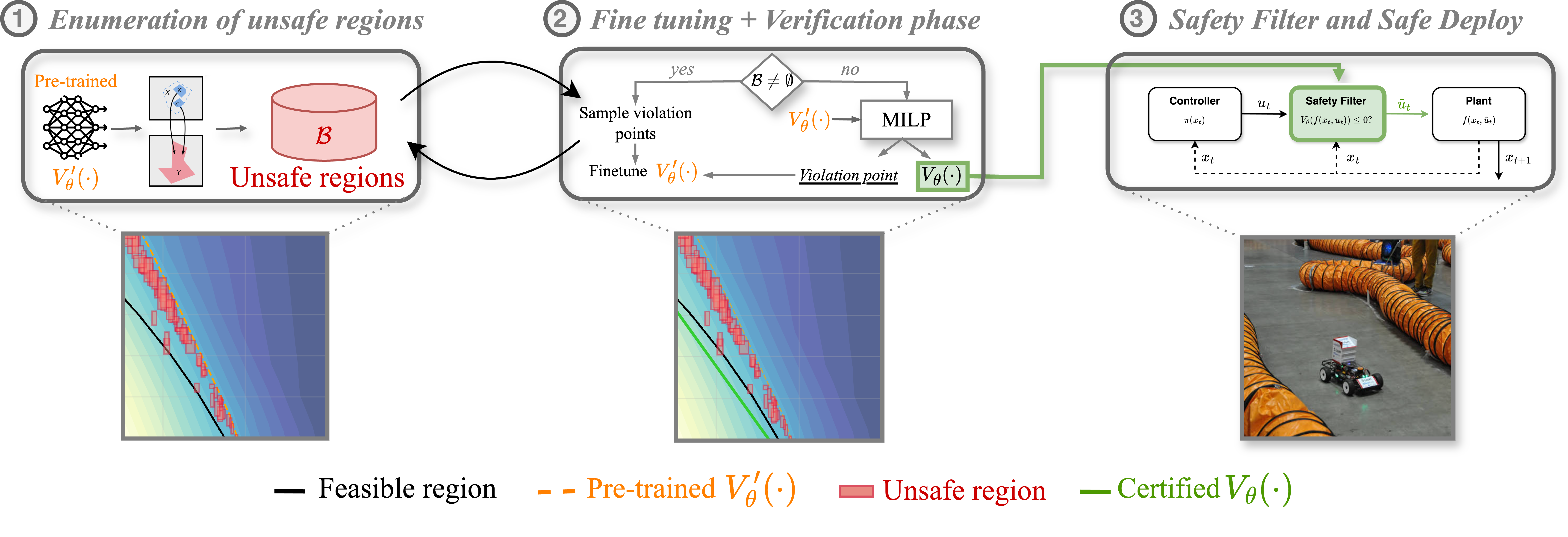}
    \caption{Overview of the proposed enumeration-guided training and online safety filter pipeline.}
    \label{fig:overview}
\end{figure*}





To address this limitation, several recent works have proposed neural HJ certificates, leveraging function approximation to scale reachability analysis to more complex systems~\cite{NeuralHJ1,NeuralHJ2}. Nevertheless, learned certificates do not inherently guarantee that the resulting zero-sublevel set defines a sound feasible region and therefore require formal verification to confirm that safety guarantees are preserved. This limitation motivated the approach of Yang. et al.~\cite{YangHWLL25}, which integrates neural value function learning within a training loop guided by adversarial search and formal verification. While this approach improves scalability and enables learning of provably safe certificates, it faces three critical challenges that we aim to address in this work. 
First, the adversarial procedure is empirical and relies on sparse, locally generated counterexamples as gradient-based attacks explore only a narrow neighborhood of the current iterate and may therefore miss failures located in other regions of the state space. As a result, the search may report no violation even when unsafe states remain in the fine-tuned certificate, not because these states are genuinely safe, but because the search did not reach them. This false negative prematurely triggers the formal verifier, which, being NP-hard~\cite{reluplex}, can incur substantial computational cost. The verifier then identifies the violation that the adversarial search missed, generates additional counterexamples, and initiates another fine-tuning cycle. This process can repeat over many iterations, leading to significant verification overhead before convergence. The inefficiency becomes particularly acute in high-dimensional systems, where the local coverage of adversarial search can be small relative to the global extent of the failure set. Second, even after a verified certificate has been obtained, the question of how to exploit it as a practical runtime safety filter for a learned controller operating in a complex environment remains unaddressed: translating the certified feasible region into an effective shielding mechanism requires additional design choices that the learning procedure alone does not resolve.

This paper addresses both limitations by introducing a novel enumeration-guided training strategy depicted in Fig. \ref{fig:overview}. At the core of our approach is the enumeration of compact unsafe regions, which replaces purely local counterexample search with a more structured exploration of the state space based on preimages of the unsafe set in deep neural networks~\cite{CountingProve,KothaBKD023}. The underlying intuition is that isolated local violations provide only limited information about the geometry of the failure set, whereas counterexamples drawn from a globally organized, spatially distributed exploration better characterize its shape and extent, offering a richer training signal to the learning algorithm. Crucially, our enumeration-guided training is not a substitute for formal verification; it improves the quality, diversity, and coverage of the counterexample set during training, while verification retains its role as the provable safety certification. In addition to this new training strategy, and building on the least-restrictive safety framework of Fisac et al.~\cite{DBLP:journals/tac/FisacAZKGT19}, we introduce a shielding mechanism that leverages the refined neural HJ certificate to enforce safety at deployment time. The proposed shield acts as a safety filter between a nominal controller and the plant: at each time step, it checks whether the intended action preserves forward invariance of the certified safe set and, if necessary, overrides it with a safe recovery action provided by a backup policy based on Linear Quadratic Regulation (LQR)~\cite{DBLP:journals/automatica/BemporadMDP02}. Unlike existing HJ-based safety filters, which typically rely on grid-based value functions or learned approximations without formal safety guarantees, our approach introduces a formally verified neural HJ certificate that, by construction, ensures forward invariance of the certified safe set under the verified system dynamics. This enables the deployment of formally certified safety filters in complex systems where closed-form safety certificates are unavailable.

Our empirical evaluation spans four safe-control verification benchmarks from~\cite{YangHWLL25}, covering systems with both linear and nonlinear dynamics, as well as a new obstacle-avoidance benchmark based on the RoboRacer platform (formerly F1TENTH)~\cite{OKelly2020F1TENTH}. Enumeration-guided training yields faster convergence times and, in general, larger certified regions across all benchmarks, with the largest gains on higher-dimensional and more complex tasks. The verified RoboRacer certificate is first evaluated in a realistic, physics-based Isaac Sim environment to assess its robustness against different track geometries, corridors, and heading disturbances. We further validate its effectiveness under real-world conditions by deploying it on a physical RoboRacer platform on an unseen real racing circuit. This deployment confirms the effectiveness of the proposed solution for complex safety-critical tasks.

\section{Background}

\subsection{Neural Hamilton-Jacobi Certification}
To synthesize formally verified neural approximations of Hamilton-Jacobi value functions, we follow the verification-guided learning pipeline introduced in~\cite{YangHWLL25}. 

\begin{definition}[Closed-loop system and safety specification]
Let $x \in \mathcal{X} \subseteq \mathbb{R}^n$ denote the system state. We consider a discrete-time closed-loop dynamical system
\begin{equation}
    x_{t+1} = f^\pi(x_t),
\end{equation}

where $f^\pi : \mathcal{X} \to \mathcal{X}$ represents the system dynamics under a fixed control policy $\pi$. Let $h : \mathcal{X} \to \mathbb{R}$ be a safety function, and define the safe set as
\begin{equation}
    \mathcal{S} := \{ x \in \mathcal{X} \mid h(x) \leq 0 \}
\end{equation}

\end{definition}

The objective is to learn a function whose zero-sublevel set characterizes a certified safe region of the state space.

\begin{definition}[Neural value function and zero-sublevel set]
Let $V_\theta : \mathcal{X} \to \mathbb{R}$ be a neural network parameterized by $\theta$. The associated zero-sublevel set is defined as
\begin{equation}
    \mathcal{C}_\theta := \{ x \in \mathcal{X} \mid V_\theta(x) \leq 0 \}
\end{equation}

\end{definition}

The goal is to synthesize $V_\theta$ such that $\mathcal{C}_\theta$ is a \emph{feasible region}, i.e., a subset of the safe set that is forward invariant under the closed-loop dynamics. Formally, this requires the following properties to hold for all $x \in \mathcal{X}$:
\begin{align}
    V_\theta(x) \leq 0 &\;\Rightarrow\; h(x) \leq 0, 
    &&\text{\small(constraint satisfaction)} \label{cs}\\
    V_\theta(x) \leq 0 &\;\Rightarrow\; V_\theta(f^\pi(x)) \leq 0, 
    &&\text{\small(forward invariance)} \label{fi}
\end{align}

The synthesis pipeline proceeds as follows. First, neural models of the closed-loop dynamics $f^\pi$ and the safety function $h$ are learned from sampled data. These models are then used to initialize and pre-train the value network $V_\theta$ \cite{YangHWLL25}. 
During pre-training, the value network is calibrated to shrink the initial zero-sublevel set, producing a conservative candidate region that reduces the burden of formal verification in later stages. Subsequently, the model is refined through a verification-guided fine-tuning process. Specifically, a projected gradient descent (PGD)-based counterexample search \cite{PGD} identifies violations of~\eqref{cs}--\eqref{fi}, while periodic mixed-integer linear programming (MILP) verification~\cite{MILP} attempts to formally certify the candidate set. While this pipeline enables the achievement of a formally verified HJ safety certificate, it may still suffer from limitations due to its reliance on sparse counterexamples, which provide only localized feedback to the learning process. In contrast, counterexamples obtained through more global and spatially structured exploration can better capture the failure set geometry, thereby offering richer guidance for refining the value function and enabling potentially easier post hoc verification.

\subsection{Neural Network Preimage Bounds Computation}
\label{sec:preimage}

Verifying properties~\eqref{cs} and~\eqref{fi} of a neural certificate $V_\theta$ amounts to characterizing the set of inputs that a deep neural network maps into a specified output region, a problem known as \emph{preimage computation}. Formally, given a network $f\colon\mathbb{R}^N\to\mathbb{R}$, an input domain $\mathcal{X}\subseteq\mathbb{R}^N$, and an output set $\mathcal{Y}\subseteq\mathbb{R}$, which together define a verification query $\mathcal{T} = \langle f,\mathcal{X},\mathcal{Y}\rangle$, the neural network preimage bounds computation, as known as \emph{AllDNN-Verification} problem~\cite{CountingProve,KothaBKD023,DBLP:conf/cav/WeiHMYNLL25}, asks for the \emph{exact} preimage set
$\Gamma(\mathcal{T}) \;=\; \bigl\{\,x \in \mathcal{X} \;\big|\; f(x) \in \mathcal{Y}\,\bigr\}$.
In our setting, the relevant query concerns the \emph{violation set} of
the current certificate, $\Gamma_\theta := \{x\in\mathcal{X} \mid V_\theta(x)\leq 0 \,\wedge\,
\max\{h(x),\, V_\theta(f^\pi(x))\} > 0\},$
i.e., the states at which~\eqref{cs} or~\eqref{fi} fails, so that
$V_\theta$ satisfies~\eqref{cs}--\eqref{fi} if and only if
$\Gamma_\theta=\emptyset$.
However, computing this preimage set exactly is \#P-hard~\cite{CountingProve} and is therefore computationally infeasible in general under standard complexity assumptions. To circumvent this barrier, recent works have introduced \emph{approximate} and/or \emph{probabilistic} preimage representations~\cite{eprove,premap,RFProve,DBLP:conf/ijcai/MarzariBM26}. Rather than recovering $\Gamma(\mathcal{T})$ exactly, these methods construct a finite collection $\mathcal{R} = \{r_1,\ldots,r_m\}$ of polytopes that provides a user-specified coverage ratio $c\in(0,1]$, with a prescribed confidence level in case of probabilistic approaches. In particular, with the stated confidence, the representation satisfies
$\bigcup_{i=1}^{m} r_i \subseteq \Gamma(\mathcal{T}), \text{ with } \frac{\mathrm{Vol}\left(\bigcup_{i=1}^{m} r_i\right)}
{\mathrm{Vol}\left(\Gamma(\mathcal{T})\right)}
\geq c.$

\begin{figure}[!t]
    \centering
    \includegraphics[width=0.8\linewidth]{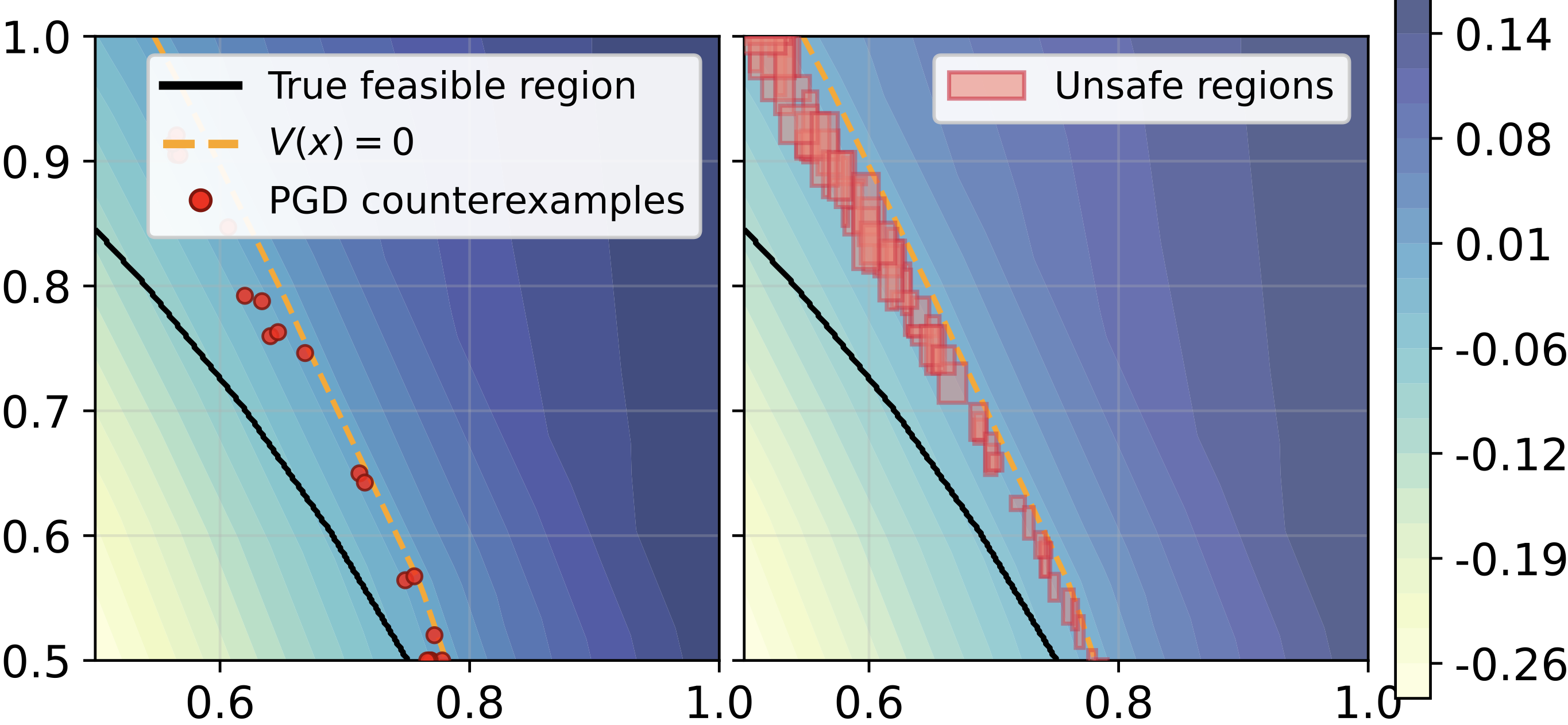}
    \caption{Comparison between PGD attack and approximate enumeration of unsafe regions (coverage $c=95\%$) for a pretrained $V(x)$ on a toy example.}
    \label{fig:enum}
\end{figure}

In the next section, we show how an enumeration-based strategy inspired by these approximate preimage methods can be embedded directly into the verification-guided training loop for neural HJ certificates. The resulting procedure replaces sparse adversarial search with a more structured exploration of the boundary of $\mathcal{C}_\theta$ (as shown in Fig. \ref{fig:enum}), while providing probabilistic coverage guarantees and thereby reducing the number of verification-guided training iterations.
Importantly, the enumeration only needs to \emph{cover} the violation
set, not to be contained in it: every state sampled from $\mathcal{R}$ is
checked exactly against~\eqref{cs}--\eqref{fi} before being used for
training, so regions that include non-violating states only reduce sampling efficiency, whereas missed regions would leave violations unaddressed. We also stress that the enumeration step is used only to guide training, not to establish the final safety guarantee: the fine-tuned certificate is still subjected to formal verification, which provides the final provable certification.

\section{Methodology}
\label{sec:methodology}

\subsection{Enumeration-guided training}
While the overall verification objective remains unchanged from that proposed in \cite{YangHWLL25}, namely enforcing constraint satisfaction and forward invariance of the zero-sublevel set of $V_\theta$, we strengthen the search for violations by introducing an explicit enumeration phase that enables more global exploration of the state space. At a high level, instead of relying solely on sparse counterexamples, we proactively identify regions of the state space that are likely to violate either condition~\eqref{cs} or~\eqref{fi}. Given the current value function $V_\theta$, this amounts to characterizing subsets of the failure set, i.e., states $x \in \mathcal{X}$ such that either $V_\theta(x) \leq 0 \land h(x) > 0$ or $V_\theta(x) \leq 0 \land V_\theta(f^\pi(x)) > 0$. However, as previously discussed, precisely identifying these regions is computationally intractable: even for small piecewise-linear neural networks, enumerating the preimages of such conditions requires reasoning over an exponential number of activation patterns, making the problem \#P-hard \cite{CountingProve}.
To address this limitation, we adopt a probabilistic enumeration strategy inspired by~\cite{RFProve}, approximating the failure set up to a desired coverage $c$ via randomized yet structured exploration of $\mathcal{X}$. Specifically, we construct an ensemble of randomized decision trees trained on bootstrapped samples of the state space, where each sample is labeled according to whether it violates constraint satisfaction or forward invariance under the current value function $V_\theta$. Each tree partitions the input space into axis-aligned regions, providing a local approximation of the failure set. By combining predictions from the entire ensemble, we detect regions that are consistently classified as violating and that together cover at least a fraction $c$ of the true unsafe volume with high confidence and bounded error \cite{RFProve}. Notably, these regions form a spatially coherent approximation of the failure set, which can be efficiently sampled to extract informative counterexamples. The resulting samples are stored in a buffer, which is used to fine-tune the value network. This process enables the learning algorithm to exploit a richer and more structured representation of violations, capturing the geometry of the failure set rather than isolated points. When the enumerated regions no longer provide sufficient valid counterexamples, we resort to a neighborhood-based fallback strategy that locally perturbs previously identified violating states, densifying the exploration around the critical boundary regions of the current zero-sublevel set $\mathcal{C}_\theta$. If no new counterexamples are found for a prescribed number of iterations, we invoke the MILP verifier to assess whether the current $V_\theta$ candidate satisfies properties~\eqref{cs}--\eqref{fi} over the entire domain $\mathcal{X}$. If certification succeeds, the procedure terminates. If a counterexample is returned, it is added to the buffer and the refinement process resumes. If the verifier is inconclusive, training proceeds without certification. 

\begin{algorithm}[!t]
\caption{Enumeration-guided fine-tuning with MILP}
\label{alg:enum_finetune}
\begin{algorithmic}[1]
\small
\Require $V, f^\pi, h, \mathcal{X}, K, \mathcal{B}_{min}, f_{\mathrm{enum}}, n, n_{fb}, c, s_{\max}$
\Ensure Fine-tuned value network $V$

\State $\mathcal{R} \gets \emptyset,\; \mathcal{B} \gets \emptyset,\; s \gets 0$

\For{$k = 1$ to $K$}
    \If{$k = 1$ or $k\;\%\:f_{\mathrm{enum}} = 0$}
        \State $\mathcal{R} \gets \texttt{EnumUnsafe}(V,h,f^\pi,\mathcal{X}, c)$
    \EndIf

    \If{$\mathcal{R}\neq\emptyset$}
        \State $\mathcal{B} \gets \texttt{SampleRegions}(\mathcal{R},n)$
    \EndIf

    \If{$\mathcal{B}=\emptyset$ or $|\mathcal{B}|< \mathcal{B}_{min}$}
        \State $\mathcal{B} \gets \mathcal{B} \cup \texttt{NeighbourFallback}(V,h,f^\pi,\mathcal{X}, n_{fb})$
    \EndIf

    \If{$\mathcal{B}\neq\emptyset$}
        \State $X \gets \texttt{FilterCE}(\mathcal{B},V,h,f^\pi)$
        \State $V \gets \texttt{Update}(V,X),\; s \gets 0$
    \Else \State $s \gets s + 1$
    \EndIf

    \If{$s \geq s_{\max}$}
        \State $(r,x_{\mathrm{ce}}) \gets \texttt{MILP}(V,h,f^\pi,\mathcal{X})$
        \If{$r = Certified$}
            \State \Return $V$
        \ElsIf{$r = Counterexample$}
            \State $\mathcal{B} \gets \mathcal{B} \cup \{x_{\mathrm{ce}}\},\; s \gets 0$
        \EndIf
    \EndIf
\EndFor
\State \Return $V$
\end{algorithmic}
\end{algorithm}

Algorithm~\ref{alg:enum_finetune} formalizes the proposed procedure. It takes as input an initial pretrained value network $V_\theta$, the learned closed-loop dynamics $f^\pi$, the safety function $h$, the state space $\mathcal{X}$, and enumeration with verification parameters. At each of the $K$ max iterations, the algorithm proceeds as follows. First, every $f_{\mathrm{enum}}$ steps, it recomputes candidate unsafe regions via $\texttt{EnumUnsafe}$, implemented using the probabilistic exploration described above. Then, if the buffer is insufficiently populated (i.e., fewer than $\mathcal{B}_{min}$ samples), new $n$ samples are drawn uniformly from these regions and filtered to retain only valid counterexamples. If no such samples are available, a neighborhood-based fallback search is performed, sampling $n_{fb}$ new samples, and whenever the buffer contains valid violations, the $V$ is updated. Otherwise, a counter is incremented, and when it exceeds $s_{\max}$, formal verification is triggered to either certify the current $V_\theta$ candidate or provide additional counterexamples.
By combining probabilistic global exploration with local refinement, the proposed method provides more informative supervision to guide the optimization of $V_\theta$ and enables faster verification of the final certificate.

\subsection{Runtime safety filter via neural HJ certificates}
Once a certificate $V_\theta$ satisfying~\eqref{cs}--\eqref{fi} has been
obtained, it can be deployed as a \emph{safety filter}, or \emph{shield}. Specifically, our filter wraps an arbitrary nominal policy $\pi_{\mathrm{nom}}$ and modifies its behavior only when necessary. Following~\cite{YangHWLL25}, verification is performed for a fixed backup policy $\pi_{\mathrm{b}}$ on a ReLU network approximation
$\hat f^{\pi_{\mathrm{b}}}$ of the closed-loop dynamics
$f^{\pi_{\mathrm{b}}}(x)=f(x,\pi_{\mathrm{b}}(x))$, i.e.,
\eqref{cs}--\eqref{fi} hold for $\hat f^{\pi_{\mathrm{b}}}$. As depicted in
Fig.~\ref{fig:safety_filter}, at each time step $t$ the system observes
$x_t\in\mathcal{C}_\theta$, queries the nominal action
$u_t=\pi_{\mathrm{nom}}(x_t)$, and performs a one-step lookahead with the
known dynamics $f$:
\begin{equation}
\tilde{u}_t =
\begin{cases}
u_t, & \text{if } V_\theta\!\left(f(x_t,u_t)\right) \leq 0,\\
\pi_{\mathrm{b}}(x_t), & \text{otherwise}.
\end{cases}
\label{eq:safety_filter}
\end{equation}
The nominal action is thus executed whenever its successor remains in
$\mathcal{C}_\theta$, and is otherwise replaced by the backup action.
Unlike prior HJ-based filters, which rely on value functions either
computed on grids, limiting them to low-dimensional systems
\cite{DBLP:journals/tac/FisacAZKGT19}, or learned without formal
guarantees \cite{DBLP:journals/trob/BorquezCWB24}, the invariance of
$\mathcal{C}_\theta$ under $\pi_{\mathrm{b}}$ is here formally verified,
which yields the following guarantee.

\begin{proposition}
Suppose $V_\theta$ satisfies~\eqref{cs}--\eqref{fi} for
$\hat f^{\pi_{\mathrm{b}}}$, $f^{\pi_{\mathrm{b}}}(x)=\hat f^{\pi_{\mathrm{b}}}(x)$
for all $x\in\mathcal{C}_\theta$, and $x_0\in\mathcal{C}_\theta$. Then the
system $x_{t+1}=f(x_t,\tilde u_t)$, with $\tilde u_t$ given
by~\eqref{eq:safety_filter}, satisfies
$x_t\in\mathcal{C}_\theta\subseteq\mathcal{S}$ for all $t\geq 0$.
\end{proposition}
\begin{proof}
By induction on $t$. Let $x_t\in\mathcal{C}_\theta$. If the nominal action
is retained, $x_{t+1}\in\mathcal{C}_\theta$ by the check
in~\eqref{eq:safety_filter}; otherwise
$x_{t+1}=f^{\pi_{\mathrm{b}}}(x_t)=\hat f^{\pi_{\mathrm{b}}}(x_t)\in\mathcal{C}_\theta$
by assumption and~\eqref{fi}. In both cases, $x_{t+1}\in\mathcal{S}$ by~\eqref{cs}.
\end{proof}

\begin{figure}[!t]
    \centering
    \includegraphics[width=0.8\linewidth]{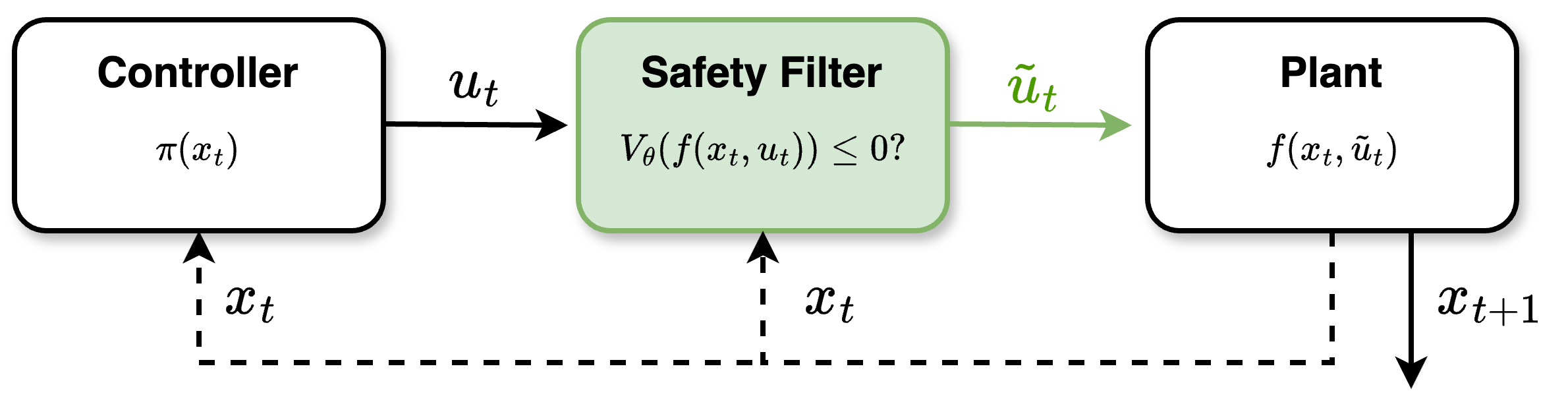}
    \caption{Block diagram of the safety filter via neural HJ certificates.}
    \label{fig:safety_filter}
\end{figure}

Both the safety check and the backup action require only forward
evaluations of $f$, $V_\theta$, and $\pi_{\mathrm{b}}$, enabling online
execution without iterative optimization. In our setting, $\pi_{\mathrm{b}}$ is a saturated LQR controller represented by a fixed ReLU network. Since the certificate is learned rather than derived analytically, the shield applies to systems for which closed-form safety certificates are unavailable. The exactness assumption reflects the verification setting
of~\cite{YangHWLL25}; since the nominal branch is checked with the known
dynamics $f$, approximation errors of $\hat f^{\pi_{\mathrm{b}}}$ affect
only the backup branch and could be absorbed by verifying~\eqref{fi} with
a Lipschitz-based margin.

\section{Empirical Evaluation}\label{sec:empirical_evaluation}

To validate the proposed enumeration-guided refinement strategy, we conducted a comprehensive evaluation across simulated and real-world scenarios. Specifically, we investigate three research questions: (RQ1) \textit{How does enumeration-guided HJ training affect the quality of the resulting safety certificate?} (RQ2) \textit{Can our approach improve training efficiency by reducing the wall-clock time required to obtain a provably verified certificate?} and (RQ3) \textit{Can the verified neural HJ certificate be effectively deployed as a runtime safety filter to enforce safety in a realistic safety-critical task?}

To answer these questions, for each task described below, we compare the state-of-the-art fine-tuning strategy of~\cite{YangHWLL25}, which uses a PGD attack to generate counterexamples followed by formal verification of the neural HJ certificate using the MILP verifier~\cite{MILP}, with our proposed enumeration-guided refinement strategy. We refer to these approaches as \texttt{PGD+MILP} and \texttt{Enum+MILP}, respectively.
We evaluate the two approaches using three key metrics. First, the \emph{true feasible rate} (TFR)~\cite{YangHWLL25} measures the
fraction of truly feasible states contained in $\mathcal{C}_\theta$,
thereby quantifying the quality of the certified region. It is estimated
on $100{,}000$ states sampled uniformly from $\mathcal{X}$, where a state
is deemed truly feasible if the closed-loop trajectory under
$\pi_{\mathrm{b}}$ starting from it incurs no constraint violation over
100 steps. 
Second, we measure the number of MILP calls and their corresponding times. For all methods, we set a fixed timeout of 150 min.
The neural HJ certificate $V_\theta$, whose zero-sublevel set $\mathcal{C}=\{x:V_\theta(x)\leq0\}$ defines the certified safe region, is represented by a neural network with two hidden layers, each containing 32 ReLU units. Under the fixed backup policy, a MILP verifier \cite{MILP} exhaustively verifies both constraint satisfaction and one-step forward invariance, i.e., conditions~\eqref{cs}--\eqref{fi}.

To account for the variability introduced by neural network initialization, we train five certificates from three independent random initializations for each method and task. Hence, we report the mean and standard deviation across these three runs, providing a more robust assessment of the performance of the compared refinement strategies. Furthermore, we note that both the baseline approach of~\cite{YangHWLL25} and our proposed enumeration-guided strategy are sensitive to hyperparameter choices. Hence, we tune each method independently using CARBS~\cite{CARBS}
under the same protocol and budget, maximizing TFR with the number of MILP calls as the cost metric; wall-clock and final verification times are recorded separately and not optimized. The experiments were conducted on AMD EPYC 9355 CPUs and an NVIDIA RTX~3090 GPU. The code and final configurations used to collect the results are available on \href{https://github.com/lmarza/CersyvePython}{\underline{github}}.

Our intuition is that our enumeration-guided fine-tuning strategy enables larger certificates while also reducing the complexity and thus computational time of the post-hoc verification. 

\textbf{Evaluation tasks.} Our empirical evaluation spans four systems from the \textit{Cersyve} verification benchmark for safe control~\cite{YangHWLL25}, together with a new obstacle-avoidance scenario, \textit{RoboRacer}, based on the F1TENTH-scale autonomous racing platform and simulation environment~\cite{OKelly2020F1TENTH}. Table~\ref{tab:benchmark_overview} summarizes the five systems considered in our experiments, while in the following, we provide the reader with a brief description of their dynamics and the corresponding safety specifications.

\begin{table}[!t]
\centering
\caption{Overview of the considered benchmarks.}
\label{tab:benchmark_overview}
\small
\setlength{\tabcolsep}{4pt}
\renewcommand{\arraystretch}{1.15}
\begin{tabular}{c c c c c}
\multicolumn{1}{c}{\textbf{Task}}
& \multicolumn{1}{c}{\shortstack[c]{\textbf{State}\\\textbf{dimension}}}
& \multicolumn{1}{c}{\shortstack[c]{\textbf{Control}\\\textbf{dimension}}}
& \multicolumn{1}{c}{\textbf{Dynamics}}
& \multicolumn{1}{c}{\textbf{Constraint}} \\
\hline
Pendulum  & 2 & 1 & Nonlinear & Linear \\
Unicycle  & 3 & 2 & Nonlinear & Nonlinear \\
RoboRacer & 4 & 1 & Linear & Linear \\
Cart Pole & 4 & 1 & Nonlinear & Linear \\
Robot Arm & 6 & 3 & Linear    & Nonlinear \\
\end{tabular}
\end{table}

\paragraph{Pendulum} The task considers stabilization of a pendulum about the upright equilibrium under angle constraints. The state is $x = [\theta,\dot{\theta}]^\top \in \mathbb{R}^2,$
where $\theta$ is the pendulum angle, and the control input is the applied torque $u \in \mathbb{R}$.
The benchmark dynamics are
\begin{equation}
\begin{aligned}
    \theta_{t+1} &= \theta_t + \dot{\theta}_{t+1}\Delta t, \\
    \dot{\theta}_{t+1} &= \dot{\theta}_t +
    \left(\frac{3g}{2l}\sin\theta_t + \frac{9u_t}{ml^2}\right)\Delta t,
\end{aligned}
\end{equation}
where $m$ is the pendulum mass, $l$ its length, $g$ the gravitational acceleration, and $\Delta t$ the sampling time. Safety is specified by the angular bounds $\theta_{\mathrm{lb}} \le \theta \le \theta_{\mathrm{ub}}$.

\paragraph{Unicycle} The task considers obstacle avoidance for a unicycle model. The state is $x = [v,x_{\mathrm{o}},y_{\mathrm{o}}]^\top \in \mathbb{R}^3,$ where $v$ denotes the forward speed and $(x_{\mathrm{o}},y_{\mathrm{o}})$ denotes the obstacle position expressed in the unicycle frame. The control input is $u = [a,\omega]^\top \in \mathbb{R}^2,$where $a$ is the longitudinal acceleration and $\omega$ the angular velocity. The benchmark dynamics are
\begin{equation}
\begin{aligned}
    v_{t+1} &= v_t + a_t\Delta t, \\
    x_{\mathrm{o},t+1} &= (x_{\mathrm{o},t} - v_t\Delta t)\cos(\omega_t\Delta t)
    + y_{\mathrm{o},t}\sin(\omega_t\Delta t), \\
    y_{\mathrm{o},t+1} &= y_{\mathrm{o},t}\cos(\omega_t\Delta t)
    - (x_{\mathrm{o},t} - v_t\Delta t)\sin(\omega_t\Delta t).
\end{aligned}
\end{equation}
Safety is encoded through the obstacle-avoidance constraint $\sqrt{x_{\mathrm{o}}^2 + y_{\mathrm{o}}^2} \ge r_{\mathrm{o}},$ where $r_{\mathrm{o}}$ is the obstacle radius.

\paragraph{Cart Pole} The task requires balancing an inverted pendulum mounted on a cart while keeping both the cart position and the pole angle within prescribed bounds. The state is $x = [y,\theta,\dot{y},\dot{\theta}]^\top \in \mathbb{R}^4,$ where $y$ is the cart position and $\theta$ is the pole angle. The control input is the horizontal force $u = F \in \mathbb{R}.$
The benchmark dynamics are
\begin{equation}
\begin{aligned}
    y_{t+1} &= y_t + \dot{y}_t\Delta t, \quad
    \theta_{t+1} = \theta_t + \dot{\theta}_t\Delta t, \\
    \dot{y}_{t+1} &= \dot{y}_t +
    \frac{F_t + ml\dot{\theta}_t^2\sin\theta_t - ml\ddot{\theta}_t\cos\theta_t}{M}\Delta t, \\
    \dot{\theta}_{t+1} &= \dot{\theta}_t + \ddot{\theta}_t\Delta t,
\end{aligned}
\end{equation}
with
\begin{equation}
    \ddot{\theta}_t =
    \frac{3Mg\sin\theta_t - 3\left(F_t + ml\dot{\theta}_t^2\sin\theta_t\right)\cos\theta_t}
    {(4M - 3m\cos^2\theta_t)l},
\end{equation}
where $m$ is the pole mass, $M$ is the total mass of pole and cart, and $l$ is the pole length. Safety is specified by
\[
y_{\mathrm{lb}} \le y \le y_{\mathrm{ub}},
\qquad
\theta_{\mathrm{lb}} \le \theta \le \theta_{\mathrm{ub}}.
\]

\paragraph{Robot Arm} The task considers a planar three-link manipulator that must reach a target while avoiding collision with a wall in front of the arm. The state is $x = [\theta_1,\theta_2,\theta_3,\dot{\theta}_1,\dot{\theta}_2,\dot{\theta}_3]^\top \in \mathbb{R}^6,$
where $\theta_1$ is the first-joint angle and $\theta_2,\theta_3$ are incremental joint angles relative to the preceding joints. The control input is
$u = [\ddot{\theta}_1,\ddot{\theta}_2,\ddot{\theta}_3]^\top \in \mathbb{R}^3,$ which represents the joint angular accelerations. The benchmark dynamics are obtained via forward-Euler discretization, and we refer the reader to the original paper \cite{YangHWLL25} for their explicit formulation. Safety is encoded by constraining the horizontal projection of the arm,
\begin{equation}
    \sum_{i=1}^{3}\cos\!\left(\sum_{j=1}^{i}\theta_j\right) \le 1.5,
\end{equation}
which prevents collision with the front wall.

\paragraph{RoboRacer overtaking}
\label{sec:roboracer}
The task considers path following and obstacle avoidance for a
1:10-scale autonomous car. The certificate state is
$x=[e_y,e_\psi,y_{\mathrm{left}},y_{\mathrm{right}}]^\top$,
where $e_y$ and $e_\psi$ denote the lateral and heading errors
relative to the path, and $y_{\mathrm{left}}$ and
$y_{\mathrm{right}}$ are the available collision-free room for the
vehicle center on either side, accounting for the vehicle footprint.
We use a curvature-compensated, small-angle approximation of the
kinematic bicycle model~\cite{Rajamani2012} in path coordinates,
with $e_y'=e_\psi$ and $e_\psi'=u$, where primes denote derivatives
with respect to vehicle arc length. The control input $u$ is the
residual curvature, constrained by $|u|\le0.31\,\mathrm{m}^{-1}$.
The model is discretized with $\Delta s=0.05\,\mathrm{m}$, and
the corridor widths are held constant in the verified model.
Safety requires
\begin{equation}
    -y_{\mathrm{right}}\le e_y\le y_{\mathrm{left}},
    \qquad |e_\psi|\le0.40.
    \label{eq:roboracer_safety}
\end{equation}

\subsection{Answer to RQ1 and RQ2 - Simulation results}
Table~\ref{tab:benchmarks_by_task} reports our simulation results. In general, the proposed enumeration-guided fine-tuning strategy enables larger TFR certificates with fewer verification calls and less computational time. This is due to the fact that, unlike PGD, which returns individual adversarial states and thus provides sparse, pointwise supervision, our enumerator identifies entire unsafe regions from which multiple violations can be sampled during fine-tuning. This exposes the value network to a broader portion of the failure set at each refinement step and reduces the risk of repeatedly correcting isolated counterexamples while nearby violations remain unresolved. 
Nonetheless, CartPole results clarify the current limitation of our enumeration approach. Its nonlinear dynamics induce curved boundaries (as shown in Fig. \ref{fig:pendulum}) that are only coarsely represented by the axis-aligned hyperrectangles used by our enumerator. Consequently, the returned regions may contain substantial portions that are uninformative for refinement, reducing the efficiency of the region-level supervision and leading to a smaller final certified region. This limitation concerns the particular geometric abstraction adopted in our implementation, rather than the enumeration-guided refinement principle itself, which is agnostic to the representation used to characterize unsafe regions.
\begin{figure}[!h]
   \centering
   \includegraphics[width=0.45\linewidth]{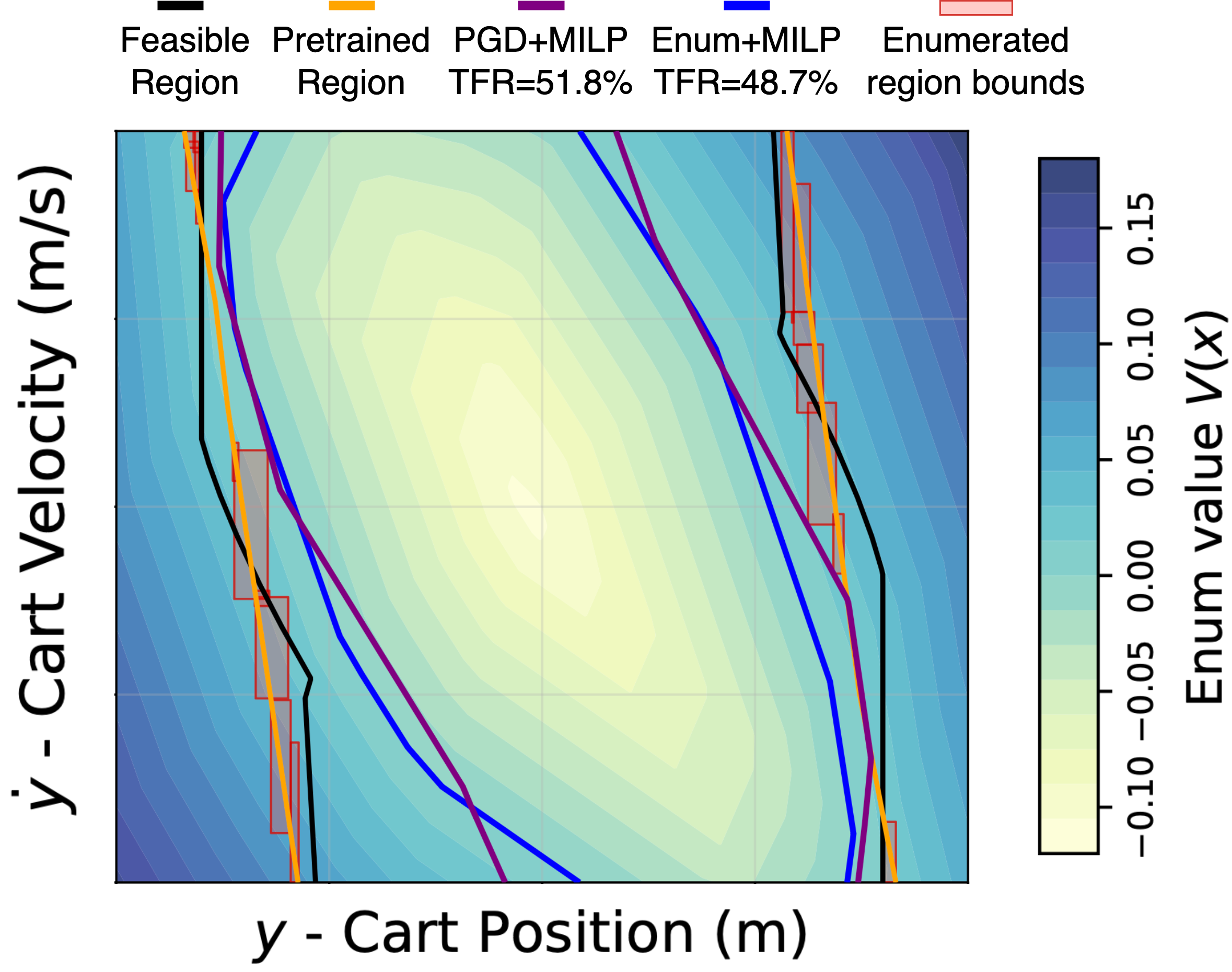}
   \caption{Comparison of best-seed certified regions for \texttt{PGD+MILP} and \texttt{Enum+MILP} on the CartPole task.}
   \label{fig:pendulum}
\end{figure}

To assess the impact of the fine-tuning strategy on the final certification step, Table~\ref{tab:benchmarks_by_task} reports the corresponding MILP workload. The experimental results support our intuition that enumeration-guided fine-tuning not only improves the quality of the learned HJ certificates but also produces certificates that are easier to verify. In fact, \texttt{Enum+MILP} consistently requires fewer or an equal number of MILP calls and achieves substantially lower overall wall-clock times across all benchmarks. Moreover, the final MILP verification time is reduced in most cases. The only exception is Robot Arm, where \texttt{Enum+MILP} certifies a substantially larger feasible region ($69.87\%$ vs.\ $49.02\%$ TFR), potentially increasing the complexity of the final verification step. Overall, these findings suggest that enumeration-guided fine-tuning improves not only the quality of the resulting certificates but also their computational verifiability.

\begin{table*}[!t]
\centering
\caption{PGD+MILP vs.\ Enum+MILP on Cersyve \cite{YangHWLL25} and RobotRacer benchmarks. Results are averaged over three random seeds.} 
\label{tab:benchmarks_by_task}
\renewcommand{\arraystretch}{1.16}
\setlength{\tabcolsep}{2.2pt}
\scriptsize
\resizebox{\textwidth}{!}{%
\begin{tabular}{lccc||ccc}
& \multicolumn{3}{c||}{\texttt{PGD+MILP} \cite{YangHWLL25}}
& \multicolumn{3}{c}{\texttt{Enum+MILP} (Ours)} \\
\cline{1-7}
\textit{Task}
& TFR
& \#MILP Calls -- Time (s)
& Wall-clock time (s)
& TFR
& \#MILP Calls -- Time (s)
& Wall-clock time (s) \\
\hline
\textit{Pendulum}
& $96.94\pm0.15$
& $1.0\pm0.0 - 4.2\pm0.1$
& $219\pm39$
& $\mathbf{97.39\pm0.35}$
& $\mathbf{1.0\pm0.0 - 2.6\pm0.5}$
& $\mathbf{108\pm32}$ \\

\textit{Unicycle}
& $94.26\pm1.15$
& $1.3\pm0.6 - 77.0\pm18.8$
& $5793\pm6569$
& $\mathbf{94.59\pm0.73}$
& $\mathbf{1.0\pm0.0 - 52.1\pm3.5}$
& $\mathbf{1287\pm823}$ \\

\textit{RoboRacer}
& N/A
& -- 
& \text{Timeout}
& $\mathbf{61.62\pm10.47}$
& $\mathbf{6.33\pm1.15 - 8.99\pm1.70}$
& $\mathbf{1729\pm21.34}$ \\

\textit{Cart Pole}
& $\mathbf{45.94\pm4.70}$
& $4.67\pm2.52 - 22.33\pm26.07$
& $8326\pm1897$
& $45.02\pm4.81$
& $\mathbf{2.00\pm0.00 - 5.08\pm1.26}$
& $\mathbf{1970\pm1058}$ \\

\textit{Robot Arm}
& $49.02\pm8.29$
& $9.0\pm10.4 - 47.0\pm14.2$
& $7955\pm10340$
& $\mathbf{69.87\pm2.21}$
& $\mathbf{7.0\pm2.6 - 69.5\pm17.4}$
& $\mathbf{1732\pm404}$ \\
\hline
\end{tabular}%
}
\end{table*}

\subsection{Ablation study - On the impact of enumeration}

To assess the contribution of identifying compact representations of counterexamples in reducing the number of steps to convergence (i.e., \textit{iterations}), we introduce a \texttt{Random Sampling+MILP} baseline that replaces unsafe-region enumeration with uniform sampling. Specifically, at each fine-tuning iteration, the baseline samples $n$ states uniformly from $\mathcal{X}$ and retains only those violating \eqref{cs} or \eqref{fi}, following the same counterexample filtering procedure (FilterCE) as \texttt{Enum+MILP}. Both methods use identical buffer sizes, update rules, and MILP verification schedules, with a maximum budget of $5000$ iterations. Thus, they differ only in the strategy used to generate candidate counterexamples.

As shown in Table~\ref{tab:ablation_sampling}, \texttt{Enum+MILP} achieves higher TFR while reducing the number of iterations required to obtain a verified certificate. We evaluate this effect on two benchmarks of increasing complexity: Pendulum and Robot Arm. On Pendulum, \texttt{Enum+MILP} reaches verification in $1563\pm111$ iterations, compared to $1934\pm506$ for random sampling. The advantage becomes particularly pronounced on Robot Arm, where uniform sampling fails to obtain a verified certificate within the $5000$-iteration budget, whereas \texttt{Enum+MILP} successfully converges in $2129\pm81$ iterations, achieving a TFR of $69.14\pm0.60$. These findings demonstrate that explicitly identifying unsafe regions provides a more informative refinement signal than uniformly sampled counterexamples, accelerating convergence and improving the effectiveness of the fine-tuning process, particularly on more complex benchmarks.

\begin{table}[!h]
\centering
\caption{Ablation on the source of counterexamples.}
\label{tab:ablation_sampling}
\renewcommand{\arraystretch}{1.2}
\setlength{\tabcolsep}{4pt}
\begin{tabular}{lcc||cc}

& \multicolumn{2}{c||}{\texttt{Random Sampling+MILP}}
& \multicolumn{2}{c}{\texttt{Enum+MILP} (Ours)} \\
\cline{2-5}
\textit{Task} & TFR & \shortstack{Iterations}
& TFR & \shortstack{Iterations} \\
\hline
\textit{Pendulum}
& $96.75\pm0.30$
& $1934\pm506$
& $\mathbf{97.29\pm0.19}$ & $\mathbf{1563\pm111}$ \\

\textit{Robot Arm}
& N/A
& Timeout
& $\mathbf{69.14\pm0.60}$
& $\mathbf{2129\pm81}$ \\
\hline
\end{tabular}
\end{table}

Finally, we also investigate whether the benefits of enumeration-guided fine-tuning persist across verification backends and network sizes. We combine both fine-tuning strategies with MILP and with $\alpha,\beta$-CROWN, a state-of-the-art branch-and-bound verifier \cite{crown}, and increase the value-network size from two hidden layers of 32 ReLU units to two layers of 64 units. 

\begin{table}[!h]
\centering
\caption{Verification scalability on Pendulum Task with increasing value-network size 2x32 (first row), 2x64 (second row).}
\label{tab:verification_scalability}
\setlength{\tabcolsep}{4pt}
\renewcommand{\arraystretch}{1.1}
\begin{tabular}{ccccc}
& \multicolumn{2}{c}{\texttt{MILP}} 
& \multicolumn{2}{c}{\texttt{$\alpha,\beta$-CROWN}} \\
\hline
 Method 
& TFR & Iters.
& TFR & Iters. \\
\hline
\texttt{PGD}  &  $96.94\pm0.15$ & $2014\pm141$ & $97.04\pm0.21$ & $1866 \pm 384$ \\
\texttt{Enum} & $\mathbf{97.29\pm0.19}$ & $\mathbf{1563\pm111}$ & $\mathbf{97.25\pm 0.58}$ & $\mathbf{1704\pm 457}$ \\
\hline
 \texttt{PGD}  & $96.78\pm0.21$ & $965\pm277$ & $96.82\pm0.23$ & $906\pm195$ \\
 \texttt{Enum} & $\mathbf{97.90\pm0.12}$ & $\mathbf{496\pm50}$ & $\mathbf{97.45\pm0.29}$ & $\mathbf{609\pm139}$ \\
\hline
\end{tabular}
\end{table}

The results of this experiment show that our enumeration-guided fine-tuning is a valuable solution, enabling a verified certificate in fewer iterations under both MILP and $\alpha,\beta$-CROWN, and the advantage persists as the network size increases.



\subsection{Answer to RQ3 - HJ certificate deployment}

To assess the robustness of the verified HJ certificate under realistic operating conditions, we first evaluate it in a physics-based Isaac Sim environment integrated with ROS2, considering different track geometries, sensing uncertainties, and disturbances affecting the perceived free-space corridor and vehicle heading. Specifically, we introduce a $0.11$~s steering delay, modeled by propagating the previously issued control and evaluating the safety filter~\eqref{eq:safety_filter} at the predicted state. We deploy the verified certificate as a safety filter around a pure-pursuit nominal controller $\pi_{\mathrm{nom}}$~\cite{Coulter1992PurePursuit}, which tracks obstacle-avoidance paths generated by a planner within the available corridor. The backup policy $\pi_{\mathrm{b}}$ is implemented as an LQR that steers the vehicle toward the corridor center. When the predicted state falls outside $\mathcal{C}_\theta$ or the operating limits are exceeded, braking is applied as a recovery action.
Table~\ref{tab:deployment} compares the shielded controller against the nominal and backup controllers over $5$ runs with randomized obstacle placements and perturbations. While the nominal controller collides in all $5$ runs and the backup controller in $4$ out of $5$, the proposed shield successfully completes all runs without collisions, achieving a mean intervention rate of $58.2\pm26.0\%$ and a mean minimum clearance of $0.679\pm0.677$~m. These results demonstrate the effectiveness of the verified safety filter in maintaining collision-free operation despite sensing uncertainties, actuation delays, and environmental disturbances.
\begin{table}[!h]
  \centering
  \caption{RoboRacer simulation over $5$ randomized obstacle placements (mean $\pm$ std). Min. Clear.: minimum distance kept to obstacle, negative on collision.}
\label{tab:deployment}
  \setlength{\tabcolsep}{4pt}
  \begin{tabular}{lcccc}
  \hline
   & Collisions & Interventions (\%) & Min. Clear. (m) \\
  \hline
  Nominal & 5/5 & $0.0\pm0.0$    & $-0.067\pm0.037$  \\
  Backup  & 4/5 & $100\pm0.0$    & $-0.041\pm0.056$ \\
  \textbf{Shield (Our)}  & \textbf{0/5} & \textbf{58.2}$\pm$\textbf{26.0}  & \textbf{+0.679}$\pm$\textbf{0.677} \\
  \hline
  \end{tabular}
  \end{table}

We further validate the proposed approach by deploying the same verified certificate on a previously unseen real-world racing track using the physical RoboRacer platform. As shown in Fig.~\ref{fig:roboracer_deployment} (left), the shielded controller closely follows the nominal trajectory in the obstacle-free scenario, while successfully navigating around six obstacles and returning toward the reference path in the more challenging scenario. The left- and right-hand avoidance maneuvers illustrated in Fig.~\ref{fig:roboracer_deployment} (right) further demonstrate the flexibility of the proposed shielding mechanism; while the planner determines the passing direction, the safety filter independently verifies or overrides the resulting steering commands to preserve safety. A video demonstration of the real-world RoboRacer experiments is available \href{https://youtu.be/OEIcr31zzbw}{\underline{here}}.

\begin{figure*}[!h]
    \centering
    \includegraphics[width=\linewidth]{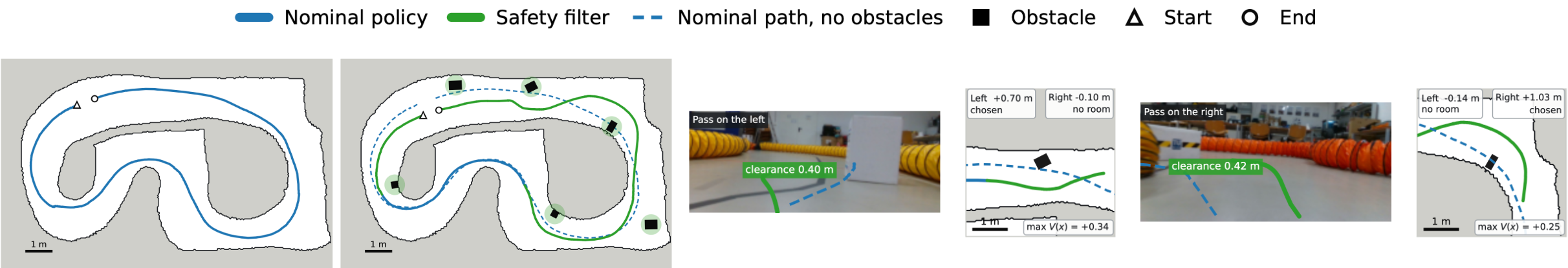}
    \caption{RoboRacer deployment of the proposed safety filter.
    (a) Trajectories showing an obstacle-free reference run
    and a run with six obstacles. The blue dashed curve denotes the planned
    path, while blue and green segments indicate nominal and safety-filter
    control, respectively.
    (b) Left- and right-passing maneuvers with camera views and corresponding
    local trajectories. The planner selects the passing side based on available
    room, while the safety filter checks steering commands.
    Negative room annotations indicate unavailable sides.
    Max $V(x)$ denotes the maximum value $V_\theta(f(x_t,u_t))$
    of the nominal successors rejected by the filter.}
    \label{fig:roboracer_deployment}
\end{figure*}
 
\section{Conclusion}

In this paper, we introduced a verification-guided training framework for neural Hamilton--Jacobi certificates based on probabilistic enumeration of unsafe state-space regions. By providing diverse, region-level counterexamples, our approach offers richer refinement guidance than local gradient-based searches, leading to larger certified HJ value function regions, while achieving substantially faster convergence times on all tested benchmarks. Crucially, our approach enables the synthesis of certificates on complex tasks where pointwise counterexamples fail. Our real-world experiment demonstrated that the verified neural HJ certificate obtained using the proposed strategy can support safe obstacle avoidance while preserving nominal driving between safety interventions.
Building on these promising results, future work will focus on extending this framework to address more challenging control tasks, such as those with high-dimensional state spaces, or improving the enumeration with semantically richer geometric representations, e.g., zonotopes, oriented bounding boxes, or polyhedra.


\section*{Acknowledgments}
This research was funded in whole or in part by the Austrian Science Fund (FWF) 10.55776/ESP1944725 and 10.55776/DOC1345324.

\bibliographystyle{unsrt}  
\bibliography{references}

@article{YangHWLL25,
  author       = {Yujie Yang and
                  Hanjiang Hu and
                  Tianhao Wei and
                  Shengbo Eben Li and
                  Changliu Liu},
  title        = {Scalable Synthesis of Formally Verified Neural Value Function for Hamilton-Jacobi Reachability Analysis},
  journal      = {J. Artif. Intell. Res.},
  volume       = {83},
  year         = {2025},
  url          = {https://doi.org/10.1613/jair.1.16946},
  doi          = {10.1613/JAIR.1.16946},
}

@article{DBLP:journals/tac/FisacAZKGT19,
  author       = {Jaime F. Fisac and
                  Anayo K. Akametalu and
                  Melanie N. Zeilinger and
                  Shahab Kaynama and
                  Jeremy H. Gillula and
                  Claire J. Tomlin},
  title        = {A General Safety Framework for Learning-Based Control in Uncertain
                  Robotic Systems},
  journal      = {{IEEE} Trans. Autom. Control.},
  volume       = {64},
  number       = {7},
  pages        = {2737--2752},
  year         = {2019},
  doi          = {10.1109/TAC.2018.2876389},
}

@article{DBLP:journals/trob/BorquezCWB24,
  author       = {Javier Borquez and
                  Kaustav Chakraborty and
                  Hao Wang and
                  Somil Bansal},
  title        = {On Safety and Liveness Filtering Using Hamilton-Jacobi Reachability
                  Analysis},
  journal      = {{IEEE} Trans. Robotics},
  volume       = {40},
  pages        = {4235--4251},
  year         = {2024},
  doi          = {10.1109/TRO.2024.3454470},
}

@article{DBLP:journals/automatica/BemporadMDP02,
  author       = {Alberto Bemporad and
                  Manfred Morari and
                  Vivek Dua and
                  Efstratios N. Pistikopoulos},
  title        = {The explicit linear quadratic regulator for constrained systems},
  journal      = {Autom.},
  volume       = {38},
  number       = {1},
  pages        = {3--20},
  year         = {2002},
  doi          = {10.1016/S0005-1098(01)00174-1},
}

@article{CARBS,
  title={Tune as you scale: Hyperparameter optimization for compute efficient training},
  author={Fetterman, Abraham J and Kitanidis, Ellie and Albrecht, Joshua and Polizzi, Zachary and Fogelman, Bryden and Knutins, Maksis and Wr{\u{A}}{\l}blewski, Bartosz and Simon, James B and Qiu, Kanjun},
  journal={arXiv preprint arXiv:2306.08055},
  year={2023}
}

@article{crown,
  title={Efficient neural network robustness certification with general activation functions},
  author={Zhang, Huan and Weng, Tsui-Wei and Chen, Pin-Yu and Hsieh, Cho-Jui and Daniel, Luca},
  journal={Advances in neural information processing systems},
  volume={31},
  year={2018}
}

@inproceedings{OKelly2020F1TENTH,
  author       = {Matthew O'Kelly and
                  Hongrui Zheng and
                  Dhruv Karthik and
                  Rahul Mangharam},
  title        = {{F1TENTH}: An Open-Source Evaluation Environment for Continuous
                  Control and Reinforcement Learning},
  booktitle    = {Proceedings of the NeurIPS 2019 Competition and Demonstration Track},
  series       = {Proceedings of Machine Learning Research},
  volume       = {123},
  pages        = {77--89},
  publisher    = {{PMLR}},
  year         = {2020},
}

@techreport{Coulter1992PurePursuit,
  author       = {R. Craig Coulter},
  title        = {Implementation of the Pure Pursuit Path Tracking Algorithm},
  institution  = {Robotics Institute, Carnegie Mellon University},
  number       = {CMU-RI-TR-92-01},
  address      = {Pittsburgh, PA},
  month        = jan,
  year         = {1992},
}

@inproceedings{PGD,
  author       = {Aleksander Madry and
                  Aleksandar Makelov and
                  Ludwig Schmidt and
                  Dimitris Tsipras and
                  Adrian Vladu},
  title        = {Towards Deep Learning Models Resistant to Adversarial Attacks},
  booktitle    = {6th International Conference on Learning Representations, ({ICLR})},
  publisher    = {OpenReview.net},
  year         = {2018},
}

@article{10.1016/j.automatica.2004.08.019,
author = {Mayne, D. Q. and Seron, M. M. and Rakovi\'{c}, S. V.},
title = {Robust model predictive control of constrained linear systems with bounded disturbances},
year = {2005},
publisher = {Pergamon Press, Inc.},
volume = {41},
number = {2},
doi = {10.1016/j.automatica.2004.08.019},
journal = {Automatica},
pages = {219–224},
numpages = {6}
}

@article{Chisci2001,
author = {Chisci, L. and Rossiter, J. A. and Zappa, G.},
title = {Systems with persistent disturbances: predictive control with restricted constraints},
year = {2001},
address = {USA},
volume = {37},
number = {7},
doi = {10.1016/S0005-1098(01)00051-6},
journal = {Automatica},
pages = {1019–1028},
numpages = {10},
}

@book{rawlings2017model,
  title={Model Predictive Control: Theory, Computation, and Design},
  author={Rawlings, J.B. and Mayne, D.Q. and Diehl, M.},
  isbn={9780975937730},
  url={https://books.google.ro/books?id=MrJctAEACAAJ},
  year={2017},
  publisher={Nob Hill Publishing}
}

@inproceedings{HJ,
  author       = {Somil Bansal and
                  Mo Chen and
                  Sylvia L. Herbert and
                  Claire J. Tomlin},
  title        = {Hamilton-Jacobi reachability: {A} brief overview and recent advances},
  booktitle    = {56th {IEEE} Annual Conference on Decision and Control {CDC}},
  pages        = {2242--2253},
  publisher    = {{IEEE}},
  year         = {2017},
  doi          = {10.1109/CDC.2017.8263977},
}

@inproceedings{NeuralHJ1,
  author       = {Jaime F. Fisac and
                  Neil F. Lugovoy and
                  Vicenc Rubies{-}Royo and
                  Shromona Ghosh and
                  Claire J. Tomlin},
  title        = {Bridging Hamilton-Jacobi Safety Analysis and Reinforcement Learning},
  booktitle    = {International Conference on Robotics and Automation {ICRA}},
  pages        = {8550--8556},
  publisher    = {{IEEE}},
  year         = {2019},
  url          = {https://doi.org/10.1109/ICRA.2019.8794107},
  doi          = {10.1109/ICRA.2019.8794107},
}

@inproceedings{NeuralHJ2,
  author       = {Somil Bansal and
                  Claire J. Tomlin},
  title        = {DeepReach: {A} Deep Learning Approach to High-Dimensional Reachability},
  booktitle    = {{IEEE} International Conference on Robotics and Automation, {ICRA}},
  pages        = {1817--1824},
  publisher    = {{IEEE}},
  year         = {2021},
  url          = {https://doi.org/10.1109/ICRA48506.2021.9561949},
  doi          = {10.1109/ICRA48506.2021.9561949},
}

@inproceedings{reluplex,
  author       = {Guy Katz and
                  Clark W. Barrett and
                  David L. Dill and
                  Kyle Julian and
                  Mykel J. Kochenderfer},
  title        = {Reluplex: An Efficient {SMT} Solver for Verifying Deep Neural Networks},
  booktitle    = {Computer Aided Verification - 29th International Conference, ({CAV})},
  series       = {Lecture Notes in Computer Science},
  pages        = {97--117},
  publisher    = {Springer},
  year         = {2017},
  url          = {https://doi.org/10.1007/978-3-319-63387-9\_5},
  doi          = {10.1007/978-3-319-63387-9\_5},
}

@inproceedings{CountingProve,
  author       = {Luca Marzari and
                  Davide Corsi and
                  Ferdinando Cicalese and
                  Alessandro Farinelli},
  title        = {The {\#}DNN-Verification Problem: Counting Unsafe Inputs for Deep
                  Neural Networks},
  booktitle    = {Proceedings of the Thirty-Second International Joint Conference on
                  Artificial Intelligence, ({IJCAI})},
  pages        = {217--224},
  publisher    = {ijcai.org},
  year         = {2023},
  url          = {https://doi.org/10.24963/ijcai.2023/25},
  doi          = {10.24963/IJCAI.2023/25},
}

@inproceedings{KothaBKD023,
  author       = {Suhas Kotha and
                  Christopher Brix and
                  J. Zico Kolter and
                  Krishnamurthy Dvijotham and
                  Huan Zhang},
  title        = {Provably Bounding Neural Network Preimages},
  booktitle    = {Advances in Neural Information Processing Systems 36: Annual Conference
                  on Neural Information Processing Systems (NeurIPS)},
  year         = {2023},
  url          = {http://papers.nips.cc/paper\_files/paper/2023/hash/fe061ec0ae03c5cf5b5323a2b9121bfd-Abstract-Conference.html},
}

@inproceedings{RFProve,
  author       = {Luca Marzari and
                  Manuele Bicego and
                  Ferdinando Cicalese and
                  Alessandro Farinelli},
  title        = {On the Probabilistic Learnability of Compact Neural Network Preimage
                  Bounds},
  booktitle    = {Fortieth {AAAI} Conference on Artificial Intelligence},
  pages        = {35707--35714},
  publisher    = {{AAAI} Press},
  year         = {2026},
  doi          = {10.1609/AAAI.V40I42.40883},
}

@inproceedings{eprove,
  author       = {Luca Marzari and
                  Davide Corsi and
                  Enrico Marchesini and
                  Alessandro Farinelli and
                  Ferdinando Cicalese},
  title        = {Enumerating Safe Regions in Deep Neural Networks with Provable Probabilistic
                  Guarantees},
  booktitle    = {Thirty-Eighth {AAAI} Conference on Artificial Intelligence},
  pages        = {21387--21394},
  publisher    = {{AAAI} Press},
  year         = {2024},
  doi          = {10.1609/AAAI.V38I19.30134},
}

@article{premap,
  author       = {Xiyue Zhang and
                  Benjie Wang and
                  Marta Kwiatkowska and
                  Huan Zhang},
  title        = {{PREMAP:} {A} Unifying PREiMage APproximation Framework for Neural
                  Networks},
  journal      = {J. Mach. Learn. Res.},
  volume       = {26},
  pages        = {133:1--133:44},
  year         = {2025},
  url          = {https://jmlr.org/papers/v26/24-1297.html},
}

@inproceedings{MILP,
  author       = {Vincent Tjeng and
                  Kai Yuanqing Xiao and
                  Russ Tedrake},
  title        = {Evaluating Robustness of Neural Networks with Mixed Integer Programming},
  booktitle    = {7th International Conference on Learning Representations, ({ICLR})},
  publisher    = {OpenReview.net},
  year         = {2019},
  url          = {https://openreview.net/forum?id=HyGIdiRqtm},
}

@Inbook{Rajamani2012,
author="Rajamani, Rajesh",
title="Lateral Vehicle Dynamics",
bookTitle="Vehicle Dynamics and Control",
year="2012",
publisher="Springer US",
address="Boston, MA",
pages="15--46",
isbn="978-1-4614-1433-9",
doi="10.1007/978-1-4614-1433-9_2",
url="https://doi.org/10.1007/978-1-4614-1433-9_2"
}

@inproceedings{DBLP:conf/ijcai/MarzariBM26,
  author       = {Luca Marzari and
                  Ezio Bartocci and
                  Enrico Marchesini},
  title        = {A Survey on the Verification of Reinforcement Learning Policies},
  booktitle    = {Proceedings of the Thirty-Fifth International Joint Conference on
                  Artificial Intelligence and the 29th European Conference on Artificial
                  Intelligence, {IJCAI/ECAI}},
  pages        = {7958--7966},
  publisher    = {ijcai.org},
  year         = {2026},
  doi          = {10.24963/IJCAI.2026/883},
}

@inproceedings{DBLP:conf/cav/WeiHMYNLL25,
  author       = {Tianhao Wei and
                  Hanjiang Hu and
                  Luca Marzari and
                  Kai S. Yun and
                  Peizhi Niu and
                  Xusheng Luo and
                  Changliu Liu},
  title        = {ModelVerification.jl: {A} Comprehensive Toolbox for Formally Verifying
                  Deep Neural Networks},
  booktitle    = {Computer Aided Verification - 37th International Conference, {CAV}},
  series       = {Lecture Notes in Computer Science},
  volume       = {15932},
  pages        = {395--408},
  publisher    = {Springer},
  year         = {2025},
  doi          = {10.1007/978-3-031-98679-6\_18},
}

\end{document}